\documentclass{amsart}
\usepackage{fullpage}
\usepackage[usenames]{color}
\usepackage{epsf}
\usepackage{bm,bbm}
\usepackage{colortbl}
\usepackage{multirow}
\usepackage{amsmath,amssymb,amsthm}
\usepackage{algorithm}
\usepackage{algorithmic}
\usepackage{subfig,floatflt,wrapfig}
\usepackage{graphicx}
\usepackage{srcltx}
\usepackage[numbers,sort&compress]{natbib}
\usepackage{url}
\graphicspath{
{figures/}
}

\title{Directional Statistics on Permutations}
\author{Sergey~M. Plis}
\address{The Mind Research Network}
\email{splis@mrn.org}
\thanks{supported by NIH under grant number NCRR 1P20 RR021938}

\author{Terran Lane}
\address{The University of New Mexico}
\email{terran@cs.unm.edu}

\author{Vince~D. Calhoun}
\address{The Mind Research Network}
\email{vcalhoun@mrn.org}

\renewcommand{\vec}[1]{\mbox{\boldmath$#1$}}
\newtheorem{thm}{Theorem}
 \newtheorem{lem}{Lemma}

\begin{document}
\maketitle
\begin{abstract}
  Distributions over  permutations arise in  applications ranging from
  multi-object tracking  to ranking  of instances.  The  difficulty of
  dealing  with these  distributions is  caused by  the size  of their
  domain,  which is  factorial in  the number  of  considered entities
  ($n!$).    It  makes   the  direct   definition  of   a  multinomial
  distribution over  permutation space impractical for all  but a very
  small  $n$.   In this  work  we propose  an  embedding  of all  $n!$
  permutations for a  given $n$ in a surface  of a hypersphere defined
  in  $\mathbbm{R}^{(n-1)^2}$.   As  a  result of  the  embedding,  we
  acquire   ability  to   define  continuous   distributions   over  a
  hypersphere  with all  the benefits  of directional  statistics.  We
  provide   polynomial  time   projections   between  the   continuous
  hypersphere representation  and the $n!$-element  permutation space.
  The  framework   provides  a  way  to   use  continuous  directional
  probability  densities   and  the  methods   developed  thereof  for
  establishing densities over permutations.  As a demonstration of the
  benefits of  the framework  we derive an  inference procedure  for a
  state-space  model over permutations.   We demonstrate  the approach
  with applications.
\end{abstract}

\section{Introduction}
Since the inception of the field of computer science, there has been a
strong dichotomy  between optimization  in continuous spaces  (such as
$\mathbbm{R}^d$)  and  combinatorial  spaces  (such as  the  space  of
permutations on  $d$ objects).   While there are  computationally hard
problems in  both kinds of  spaces, combinatorial spaces are  far more
often the  villain.  It seems  as if nearly all  interesting learning,
optimization, and representation  problems in combinatorial spaces are
NP-complete  in the  best case.   Bayesian inference  in the  space of
permutations,  for   example,  is  an   important,  yet  frustratingly
difficult problem~\cite{KondorHowardJe}.

We  feel that a  key factor  at the  heart of  this dichotomy  is that
combinatorial  spaces  are   far  more  \emph{unstructured}  than  the
familiar continuous spaces.  A priori, combinatorial spaces are simply
sets of  objects, with no  relationship among them.  Compare  this to,
say,  Euclidean  $d$-space,  which  comes equipped  with  a  topology,
continuity, completeness, compact subsets, a metric, an inner product,
and so  on~\cite{mun1975a}.  On these properties are  built the entire
infrastructure    of    analysis,    including    tools    like    the
derivative~\cite{kre1978a}.  In  turn, the derivative is  at the heart
of  most  optimization  techniques  and representations  such  as  the
Fourier basis.   Essentially, the  last four centuries  of mathematics
has  been  developing tools  for  representation  and optimization  in
continuous spaces.  Combinatorial spaces, on the other hand, have been
burdened with fewer assumptions, but endowed with fewer advantages.

One strategy  for working with  combinatorial spaces is to  embed them
into continuous  spaces and work  there with powerful  analytic tools.
This  trick has  proven to  be  powerful in,  for example,  continuous
relaxations of  integer programming problems~\cite{gom1958a}.   It has
enjoyed relatively less penetration in machine learning, however.  And
where            versions             of            it            have
appeared~\cite{fligner1986distance,meila-consensus}, the connection to
the  topology  and  analytic  properties  of the  embedding  space  is
typically not made explicit, nor fully exploited.

In this paper,  we demonstrate the power of  the embedding approach by
developing  a  fast,  accurate  approach to  Bayesian  inference  over
permutations.      Arising     in     tasks     such     as     object
tracking~\cite{KondorHowardJe} or ranking~\cite{meila-consensus}, this
problem  is challenging  because  of the  factorially-large number  of
parameters  in  an  exact  representation  of  a  general  probability
distribution  in   this  space.   Prior  approaches   have  worked  by
approximating a general probability distribution with a restricted set
of  basis   functions~\cite{Huang_2009_6385},  or  by   embedding  the
permutation space  only implicitly,  and working with  a heuristically
chosen probability distribution~\cite{meila-consensus}.

The   paper   follows  the   hierarchical   structure   of  our   main
contributions,  where  each  level  of  the hierarchy  is  split  into
theoretical observations and developments that make these observations
practical:
\begin{itemize}
\item {\bf  Theoretical observations}: we demonstrate  an embedding of
  the  $n!$  permutation  set   onto  the  surface  of  a  hypersphere
  $\mathbbm{S}^d$  centered at  the origin  in  $\mathbbm{R}^{d+1}$ with
  $d=(n-1)^2-1$.
  \begin{itemize}
  \item  {\it Observations}:  we  propose a  hypersphere embedding  of
    permutations.
  \item   {\it  Practical   results}:  we   develop   polynomial  time
    transformations  between the discrete  $n!$ permutation  space and
    its continuous hypersphere representation.
  \end{itemize}
\item  {\bf  Results that  allow  practical  use  of the  theory}:  we
  demonstrate        a        bridge        between        directional
  statistics~\cite{DirectionalStatistics}  and  permutation sets  that
  leads to efficient inference.
  \begin{itemize}
  \item {\it  Observations}: we  propose the von  Mises-Fisher density
    over permutations.
  \item {\it  Practical results}: we develop  efficient inference over
    permutations in a state-space model.
    \begin{itemize}
    \item We employ analytical product and marginalization operations.
    \item  We  show  efficient  transformation of  partially  observed
      permutations    onto   the    surface    of   the    hypersphere
      $\mathbbm{S}^d$.
    \end{itemize}
  \end{itemize}
\end{itemize}

\section{Embedding permutations onto a hypersphere surface}
\label{sec:embedding}
Among  many  representations  of  permutations  in this  work  we  are
interested  in  the  $n\times  n$  permutation  matrix  representation
$\mathbf{P}$. Note that nowhere in the  paper we are going to use this
as  the permutation  operator, which  is  the usual  intension of  the
matrix  representation   of  permutations.   The   permutation  matrix
representation   is  a   square  bistochastic   matrix   with  entries
$\mathbf{P}_{ij}\in  \{0,1\}$, serves  more  as an  easy to  interpret
guide  and  a  way  to  establish some  required  properties  than  an
expression for a linear operator,  whereas we interpret it in the rest
of  the paper  merely as  a vector  in $\mathbbm{R}^{n^2}$.   To avoid
notation clutter  we treat  all the matrices  further in the  paper as
vectors  in $\mathbbm{R}^{n^2}$ omitting  the special  vector stacking
operation symbols (such  as $vec\left(\cdot\right)$), unless specified
otherwise.
\subsection{Representation}
\label{sec:hs}
In this section we will show  how a permutation set with $n!$ elements
can  be   embedded  onto  the  surface  of   a  $(n-1)^2$  dimensional
hypersphere.

Our  representation takes advantage  of the  geometry of  the Birkhoff
polytope   and   in    part   relies   on   the   Birkhoff-von~Neumann
theorem~\cite{BAPAT}, which we state here without proof.
\begin{thm}\label{thm:BvN}
  All  $n\times  n$ permutation  matrices  in $\mathbbm{R}^{n^2}$  are
  extreme points of a  convex $(n-1)^2$ dimensional polytope, which is
  the convex hull of all bistochastic matrices.
\end{thm}

Next, we formulate a lemma that the rest of the section is based on:
\begin{lem}\label{lem:extreme}
  Extreme points of  the Birkhoff polytope are located  on the surface
  of a radius $\sqrt{n-1}$  hypersphere clustered around the center of
  mass of all $n!$ permutations.
\end{lem}
\begin{proof}
  To show that  the statement is valid we first  compute the center of
  mass  and then show  that each  permutation is  located at  an equal
  distance  from  this  center.   The  center  of  mass  for  all  the
  permutations on $n$ objects is defined in $\mathbbm{R}^{n^2}$ as
  $c_M = \frac{1}{n!}\sum_{k=1}^{n!}\mathbf{P}_k$.

  We  observe  that  the  number  of permutation  matrices  for  which
  $\mathbf{P}_{11}=1$  is $(n-1)!$, which  follows from  the effective
  removal of the first row and  column of an $n\times n$ matrix caused
  by  the assignment.   Thus,  $\sum \mathbf{P}_{11}  = (n-1)!$  which,
  following the same reasoning,  is true for any $\mathbf{P}_{ij}$ and
  leads to
  \begin{align}
    c_M &=   \frac{1}{n!}  (n-1)!    \mathbbm{1}   =  \frac{1}{n}
    \mathbbm{1}
  \end{align}
  To  see that  all permutations  are equidistant  from the  center of
  mass,   we   observe   that   $\|\mathbbm{1}  -   \mathbf{P}\|_2   =
  \sqrt{n^2-n}$  for any  $\mathbf{P}$. With  this observation  we can
  compute the radius of the sphere:
  \begin{align}
    r_{s}   &=   \left\|\frac{1}{n}\mathbbm{1}-\mathbf{P}\right\|_2  =
    \sqrt{(n^2-n)\frac{1}{n^2}+n(\frac{1}{n}-1)^2} = \sqrt{n-1}
  \end{align}
\end{proof}

To show  that the  hypersphere of Lemma~\ref{lem:extreme}  is embedded
into a  space of lower  dimension than $\mathbbm{R}^{n^2}$  we observe
the  following.    With  respect   to  the  original   formulation  of
permutations  in  $\mathbbm{R}^{n^2}$,  all  of the  permutations  are
located on  the intersection of  a hypersphere centered at  the origin
with      $\sqrt{n}$     radius      and     a      hypersphere     of
Lemma~\ref{lem:extreme}. This intersection is still a hypersphere only
with dimension  lowered by one. The  following lemma allows  us to get
the   dimension   of   this   hypersphere   down   to   the   one   of
Theorem~\ref{thm:BvN}.
\begin{lem}\label{lem:subspace}
  All  permutations $\mathbf{P}$  are located  on the  intersection of
  $2n-1$ hyperplanes,  i.e., in $(n-1)^2$-dimensional  affine subspace
  of $\mathbbm{R}^{n^2}$.
\end{lem}
\begin{proof}
  Let us denote by $\mathbf{W}_{i,\vec{1}}$ an $n\times n$ matrix with
  all elements  except a single $i^{th}$  row of ones set  to zero and
  likewise $\mathbf{W}_{\vec{1},i}$ for columns. Observe that:
  \begin{align}
    \begin{split}
      vec\left(\mathbf{W}_{i,\vec{1}}\right)^T
      vec\left(\mathbf{P}\right)         &=         1
    \end{split}
    \begin{split}
      vec\left(\mathbf{W}_{\vec{1},i}\right)^T
      vec\left(\mathbf{P}\right) &= 1
    \end{split}
  \end{align}
  for  any permutation matrix\footnote{In  fact, for  any bistochastic
    matrix,  as implied by  Theorem~\ref{thm:BvN}}.  It  follows, that
  all permutations are located  at an intersection of $2n$ hyperplanes
  defined    by    their    normals:   $\mathbf{W}_{\vec{1},i}$    and
  $\mathbf{W}_{i,\vec{1}}$, with $n\in\{1\hdots  n\}$, and having bias
  of  1.    This  set  is,  however,  not   independent,  because  any
  $\mathbf{W}_{i,\vec{1}}$ can be expressed by a linear combination of
  the  other $2n-1$  vectors by  setting weights  of $\mathbf{W}_{j\ne
    i,\vec{1}}$ to  $-1$ and weights of  $\mathbf{W}_{\vec{1},i}$ to 1
  for $i,j\in\{1\hdots  n\}$. This  leads to $2n-1$  hyperplanes whose
  intersection forms the space in which the hypersphere containing the
  Birkhoff-polytope  is  located. Thus,  the  dimension  of the  space
  containing the polytope is $n^2 - 2n +1 = (n-1)^2$.
\end{proof}

All permutation  matrices on  $n$ objects belong  to the surface  of a
radius      $\sqrt{n-1}$      hypersphere,      $\mathbbm{S}^d$,      in
$\mathbbm{R}^{(n-1)^2}$  as  established  by  Lemmas~\ref{lem:extreme}
and~\ref{lem:subspace}.  We  do not  rigorously show here,  but assume
that  by inherent symmetry  in the  structure of  permutation matrices
they are  distributed evenly across the surface  of $\mathbbm{S}^d$.

\subsection{Transformations}
\label{sec:proj}
The representation  of the  previous section allows  us to  define and
manipulate  probability  density  functions on  $\mathbbm{S}^d$  using
approaches  of  continuous  mathematics  and  only  then  transforming
quantities of  interest back to  the discrete $n!$  permutation space.
This is useful  when there is a way  to efficiently transform elements
of one space to  the other.  Next we show how this  can be achieved in
polynomial time.

The key  components posing  difficulties are discrete  vs.  continuous
space, and  the requirement  of $\mathbbm{S}^d$ to  be origin-centered
(required for Section~\ref{sec:vmf}).  The former poses a considerably
more challenging  problem than  the latter and  absence of  both would
reduce  the  required transformations  to  a  simple  change of  basis
between  $\mathbbm{R}^{n^2}$ and $\mathbbm{R}^{(n-1)^2}$.   We develop
the transformations in the proof to the following lemma.
\begin{lem}
  There  exist polynomial  time transformations  between  the discrete
  $n!$  permutation  space  and  the surface  of  the  origin-centered
  $(n-1)^2$ dimensional hypersphere of radius $\sqrt{n-1}$.
\end{lem}
\begin{proof}
  The  transformation  {\bf from  a  permutation space  to  $\mathbbm{S}^d$}
  requires only  a short sequence of  linear operations as  it is made
  clear by lemmas of Section~\ref{sec:hs}:
  \begin{enumerate}
  \item    Shift     the    permutation    matrix     $\mathbf{P}$    by
    $\frac{1}{n}\mathbbm{1}$ to put the center of mass at the origin.
  \item  Change  the  basis  by projecting into  the
    $\mathbbm{R}^{(n-1)^2}$         subspace        orthogonal        to
    $\mathbf{W}_{\vec{1},i}$     and     $\mathbf{W}_{i,\vec{1}}$.
  \end{enumerate}
  Since  there  are  $(n-1)^2$  basis  vectors of  length  $n^2$,  the
  projection  operation takes $O(n^4)$.   Note that  the basis  can be
  obtained by  the QR factorization,  which is $O(n^6)$ in  this case,
  but needs to be computed only once for a given $n$.
  
  Transforming  an arbitrary  point {\bf  from $\mathbbm{S}^d$  to the
    permutation space}  is more challenging.  Now we  have to linearly
  transform the point  from $\mathbbm{S}^d$ to $\mathbbm{R}^{n^2}$ and
  then  among  $n!$ possibilities  find  a  permutation,  that is  the
  closest, in  $L_2$ sense, to  a given point.  The  transformation is
  easily  done by  inverting the  order of  operations for  going from
  $\mathbbm{R}^{n^2}$  to $\mathbbm{S}^d$,  which amounts  to $O(n^4)$
  operations.   Let us  show  how to  efficiently  find a  permutation
  matrix closest to a transformed point.
  
  Given    an     arbitrary    point    $\mathbf{T}^\mathbbm{S}$    in
  $\mathbbm{R}^{n^2}$,    which   corresponds    to    a   point    on
  $\mathbbm{S}^d$,  as indicated  by the  superscript, we  introduce a
  matrix $\mathbf{D}$ where
  \begin{align}
    \mathbf{D}_{ij} = (\mathbf{T}_{ij}^\mathbbm{S} - 1)^2
  \end{align}
  Finding   the  permutation  $\mathbf{P}^{\mathbbm{S}}$   closest  to
  $\mathbf{T}^{\mathbbm{S}}$         amounts         to        finding
  $\mathbf{P}^{\mathbbm{S}}$               that              minimizes
  $\sum_{ij}\mathbf{D}_{ij}\mathbf{P}_{ij}$.   This  is  the  same  as
  matching every column  and each row to a  single counterpart so that
  the  sum   of  matching   weights  (elements  of   $\mathbf{D}$)  is
  minimal. In  this case, $\mathbf{D}$  is an $n\times  n$ edge-weight
  matrix  for  a $2n$  node  bipartite  graph  with $n$  elements  per
  partition.  This is the familiar minimum weighted bipartite matching
  problem~\cite{west2001introduction}.  This  observation allows us to
  apply      a      minimum      weighted      bipartite      matching
  algorithm~\cite{west2001introduction}   and  obtain   a  permutation
  $\mathbf{P}^{\mathbbm{S}}$  closest  to  $\mathbf{T}^{\mathbbm{S}}$.
  The running time of the  fastest general algorithms for solving this
  problem is $O(n^2\log{n} + n^2e)$,  where $e$ is the number of edges
  in the  bipartite graph.  Since the  number of edges in  our case is
  always $n$,  the running time effectively  becomes $O(n^3)$. However
  it is dominated by the time of projecting a point from $\mathbbm{S}^d$
  to $\mathbbm{R}^{n^2}$, which is $O(n^4)$ as shown before.
\end{proof}

Coupling  the   probability  representations  to   the  transformation
operations bridges  the gap between the  discrete, combinatorial space
of permutations and the continuous, low-dimensional hypersphere.  This
allows us to lift the  large body of results developed for directional
statistics~\cite{DirectionalStatistics}    directly   to   permutation
inference.

\section{Directional statistics}
\label{sec:vmf}
A number of probability density functions on $\mathbbm{S}^d$ have been
developed        in       the        field        of       directional
statistics~\cite{DirectionalStatistics}.  A  detailed account is given
for  an interested reader  in~\cite[Chapter 9]{DirectionalStatistics}.
The directional statistics framework allows us to define quite general
classes of  density functions over  permutations.  In the rest  of the
paper, we use one of the basic models to demonstrate the usefulness of
our representation and the model as well.

\subsection{von Mises-Fisher distribution}
This is a  $m$-variate von Mises-Fisher\footnote{Sometimes also called
  the Langevin distribution.}  (vMF) distribution of a $m$-dimensional
vector  $\vec{x}$, where  $\|\vec{\mu}\|=1$,  $\vec{\kappa}\ge 0$  and
$m\ge 2$:
\begin{align}
  \begin{split}
    f(\vec{x}|\vec{\mu},\kappa) &=
    Z_m\left(\kappa\right)e^{\kappa\vec{\mu}^T\vec{x}}
    \label{eq:vMF}  
  \end{split}
  \begin{split}
    \hspace{0.05in} &\mbox{with normalization term} \hspace{0.15in}
  \end{split}
  \begin{split}
    Z_m\left(\kappa\right) &=
    \frac{\kappa^{m/2-1}}{(2\pi)^{m/2}\bm{I}_{m/2-1}(\kappa)},
  \end{split}
\end{align}
where $\bm{I}_r(\cdot)$ is the $r^{th}$ order modified Bessel function
of  the first  kind  and $\vec{\kappa}$  is  called the  concentration
parameter.    Examples   of   samples   from   the   distribution   on
$\mathbbm{S}^2$ are shown in Figure~\ref{fig:vMFsamples}.
\begin{figure}[ht!]
  \includegraphics[width=1\textwidth]{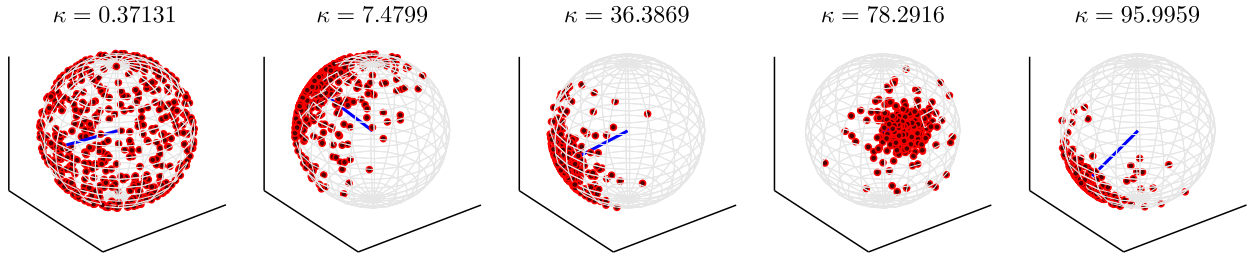}
  \caption{Samples  of  the   von  Mises-Fisher  density  function  on
    $\mathbbm{S}^2$ for random $\vec{\mu}$ and $\kappa$.}
  \label{fig:vMFsamples}
\end{figure}

In terms of a pdf on permutations the vMF establishes a distance-based
model, where distances are  geodesic on $\mathbbm{S}^d$. The advantage
of the  formulation in a  continuous space is  the ability to  apply a
range of  operations on the  pdf and still  end up with the  result on
$\mathbbm{S}^d$.   This   advantage  is  realized   in  the  inference
procedures which we establish next.

\subsection{Efficient inference in a state space model}
The  results presented  above establish  a  framework in  which it  is
possible to define and manage in reasonable time probability densities
over permutations.   An important application of this  framework is in
the probabilistic data association (PDA)~\cite{PDAFRasmussPAM}. In PDA
we  are interested  in maintaining  links between  objects  and tracks
under the noisy tracking conditions.  Ignoring the underlying position
estimation  problem we  focus  on  the part  related  to the  identity
management, as in~\cite{KondorHowardJe},  which boils down to tracking
a  hidden permutation  (identity  assignment) under  a noisy  observed
assignment.

In order to  perform identity tracking of permutations  in the context
of recursive Bayesian filtering (which we  are going to do) we need to
define the following components:
\begin{enumerate}
\item     A      transition     model,    $P\left(\mathbf{X}_t     |
    \mathbf{X}_{t-1}\right)$;
\item    An     observation    model,    $P\left(\mathbf{Y}_t    |
    \mathbf{X}_{t}\right)$  where  $\mathbf{Y}_{t}$  is the  noisy
  observation of the hidden permutation matrix $\mathbf{X}_{t}$;
\item A way to perform the following operations:
  \begin{align}\label{eq:product}
    \mbox{multiplication:} \hspace{0.5cm} &
    P(\mathbf{X}_t|\mathbf{Y}_t)                          \propto
    P(\mathbf{Y}_t|\mathbf{X}_t)
    P(\mathbf{X}_t|\mathbf{Y}_{t-1})      
    \\\label{eq:projection}  
    \mbox{marginalization:} \hspace{0.5cm} &
    P(\mathbf{X}_t|\mathbf{Y}_{t-1})  =
    \int                               P(\mathbf{X}_t|\mathbf{X}_{t-1})
    P(\mathbf{X}_{t-1}|\mathbf{Y}_{t-1})    d\mathbf{X}_{t-1} 
  \end{align}
\end{enumerate}

Avoiding  transformation overhead  we  restrict all  of  the above  to
$\mathbbm{S}^d$.    Hence,  $\mathbf{X}$  and   $\mathbf{Y}$  are
$\mathbbm{S}^d$  representations   of  their  respective   hidden  and
observed  permutations.   We define  both  transition and  observation
models  as vMF  functions centered  at the  true permutation.   Due to
similarity of the  vMF model to the multivariate  Gaussian density, it
seems  natural to  view this  recursive filter  as an  analogy  of the
Kalman filter.  In this view, the result of this sections is porting a
widely  successful tracking  model  to the  discrete $n!$  permutation
space.

To further  stress the  analogy with the  Kalman filter, we  show that
projection operation can be computed analytically in a closed form and
marginalization  operation can be  efficiently approximated  with good
accuracy~\cite{chiuso1998visual,DirectionalStatistics}.             For
observation      model      $P(\mathbf{Y}_t|\mathbf{X}_t)      \propto
vMF(\mathbf{Y}_t,\kappa_{obs})$        and       posterior       model
$P(\mathbf{X}_t|\mathbf{Y}_{t-1})                               \propto
vMF(\vec{\mu}_{pos},\kappa_{pos})$    the   multiplication   operation
results in a vMF for $P(\mathbf{X}_t|\mathbf{Y}_{t})$ parametrized as
\begin{align}
  \begin{split}
    \vec{\mu}_t &= \frac{1}{\kappa}\left(\kappa_{obs}\mathbf{Y}_t +
      \kappa_{pos}\vec{\mu}_{pos}\right)
  \end{split} &
  \begin{split}
    \kappa_t &= \|\kappa_{obs}\mathbf{Y}_t + \kappa_{pos}\vec{\mu}_{pos}\|.
  \end{split}
\end{align}
In  the case of  a vMF  transition model,  the marginalization  can be
performed with a  reasonable accuracy and speed using  the fact that a
vMF  can  be  approximated  by  an  angular  Gaussian  and  performing
analytical convolution of  angular Gaussian with subsequent projection
back   to  vMF   space~\cite{DirectionalStatistics}.    Resulting  vMF
$P(\mathbf{X}_t|\mathbf{Y}_{t-1})$ is parametrized as:
\begin{align} 
  \begin{split}
    \vec{\mu} &= \mathbf{X}_{t-1} + \vec{\mu}_{pos}
  \end{split}
  \begin{split}
    \kappa &= A_d^{-1}(A_d(\kappa_{pos})A_d(\kappa_{tr}))
  \end{split} &
  \begin{split}
    A_d(\kappa) &= \frac{\bm{I}_{d/2}(\kappa)}{\bm{I}_{d/2-1}(\kappa)}
  \end{split}
\end{align}
The ratio of modified Bessel  functions required for this approach can
be efficiently computed with high accuracy by using Lentz method based
on evaluating continued fractions~\cite{lentz1976generating}.

\subsubsection{Partial observations}
\label{sec:partobs}
\begin{wrapfigure}{r}{0.38\columnwidth}
  \centering
  \includegraphics[width=0.35\columnwidth]{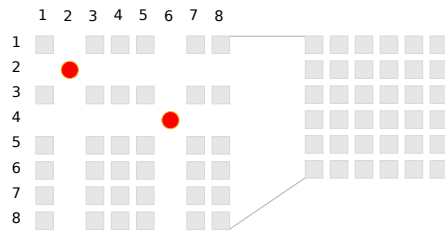}
  \caption{An example of a fallback to a lower dimensional permutation
    space when a partial observation becomes available.}
  \label{fig:fallback}
\end{wrapfigure}
Analytical computation  of the Bayesian  recursive filtering presented
above relies on the fact that permutations are observed completely. In
tracking  problems  that  would  mean  the algorithm  has  to  receive
observations (up to noise) of identities of every tracked object. This
is a  rare setting and  most commonly observations are  available only
partially.

When  a partial  observation  of $o$  objects  becomes available,  the
dimension  of the unknown  part of  $\mathbf{Y}$ is  reduced from
$n^2$   to   $(n-o)^2$.   The   mechanism   of   this   is  shown   in
Figure~\ref{fig:fallback}, where circles indicate two observed objects
and squares  indicate the unknown parts of  $\mathbf{P}$.  The unknown
part of the representation of $\mathbf{P}$ on $\mathbbm{S}^d$ needs to
be    marginalized    out    to    obtain    the    likelihood    used
in~\eqref{eq:product}.   Figure~\ref{fig:fallback}   shows  that  this
marginalization  is   straightforward  in  $\mathbbm{R}^{n^2}$  space.
Unfortunately,   to   implement~\eqref{eq:projection},   we  need   to
marginalize  on  the  surface  of  the  sphere,  $\mathbbm{S}^d\subset
\mathbbm{R}^{(n-1)^2}$ -- a much more difficult task.

Denoting the orthogonal part  of the basis in $\mathbbm{R}^{n^2}$ that
represents  the  $\mathbbm{R}^{(n-1)^2}$ subspace  by  an $n^2  \times
(n-1)^2$ matrix $\mathbf{Q}$, we project into this subspace by:
\begin{align}
  vec\left(\mathbf{Y}\right) &= \mathbf{Q}^Tvec\left(\mathbf{P -
      \frac{1}{n}\mathbbm{1}}\right).
\end{align}

In the  case of a partial  observation, we know which  elements of the
vector being projected are consistent with the observation and are not
going to change  and which elements can have  any possible value. This
allows us to split the resulting vector $\mathbf{Y}$ into
\begin{align}
  \mathbf{Y} = \mathbf{Y}_{*} + \mathbf{Y}_{?},
\end{align}
where  $\mathbf{Y}*$ and  $\mathbf{Y}_?$  respectively denote  the
observed and unobserved parts.

The likelihood with the unknown observations marginalized out becomes:
\begin{align}
  \frac{1}{Z} \int_{\mathbf{Y}_?}e^{\kappa_1 \mathbf{Y}_*^T\vec{x}
    +                                                          \kappa_1
    \mathbf{Y}_?^T\vec{x}}d\mathbf{Y}_?
  &=                      \frac{1}{Z}                      e^{\kappa_1
    \mathbf{Y}_*^T\vec{x}}
  \int_{\mathbf{Y}_?}             e^{\kappa_1\mathbf{Y}_?^T\vec{x}}
  d\mathbf{Y}_?
  \label{eq:margvMF}
\end{align}

Some  details make  computing the  integral  in~\eqref{eq:margvMF} not
totally  trivial: $\vec{x},\mathbf{Y}_*$,  and  $\mathbf{Y}_?$ are  of
different  length;  although $\vec{x}$  is  fixed, $\mathbf{Y}_*$  and
$\mathbf{Y}_?$  are  not  allowed   to  take  any  possible  angle  in
$\mathbbm{R}^{(n-1)^2}$. We omit the details of the derivation dealing
with these difficulties and just state the parameters of the resulting
vMF likelihood function:
\begin{align}
  \begin{split}
  \vec{\mu} &= \frac{\mathbf{Y}_*}{\| \mathbf{Y}_*\|_2}    
  \end{split}
  \begin{split}
  \kappa &= \|\kappa_1 \mathbf{Y}_*\|_2
  \end{split}
\end{align}

Thus, in  the case of vMF  we can execute a  recursive Bayesian filter
using  only  analytical  computation  even  in  the  cases  when  only
partially observed data is available. This makes the state space model
applicable in a  much wider range of scenarios  than our initial model
presented in Section~\ref{sec:cfo}.

\section{Experiments}
\begin{figure}[ht!]
  \centering
  \subfloat[][25 objects]{
    \includegraphics[width=0.43\textwidth]{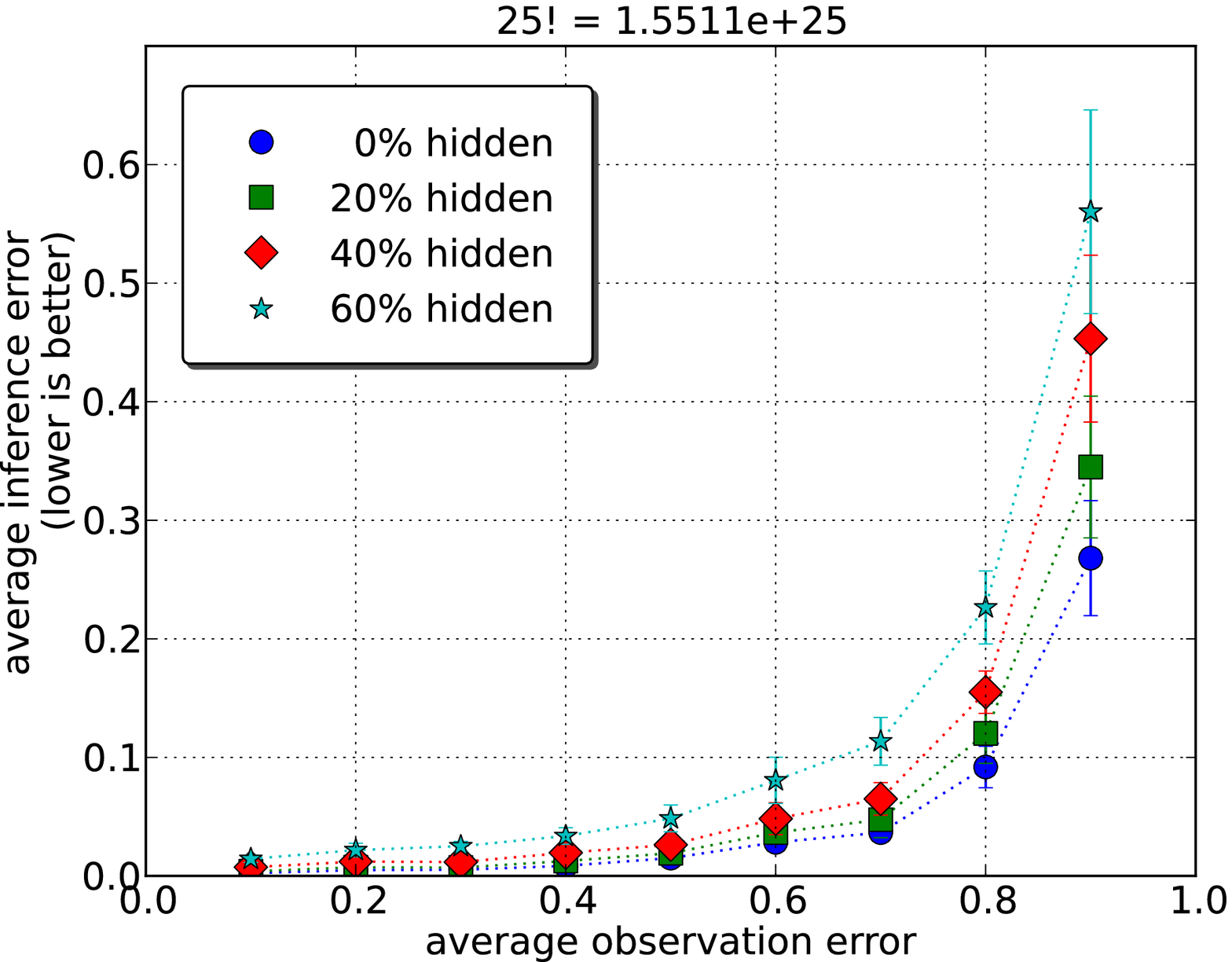}
  }
\qquad
  \subfloat[][50 objects]{
    \includegraphics[width=0.43\textwidth]{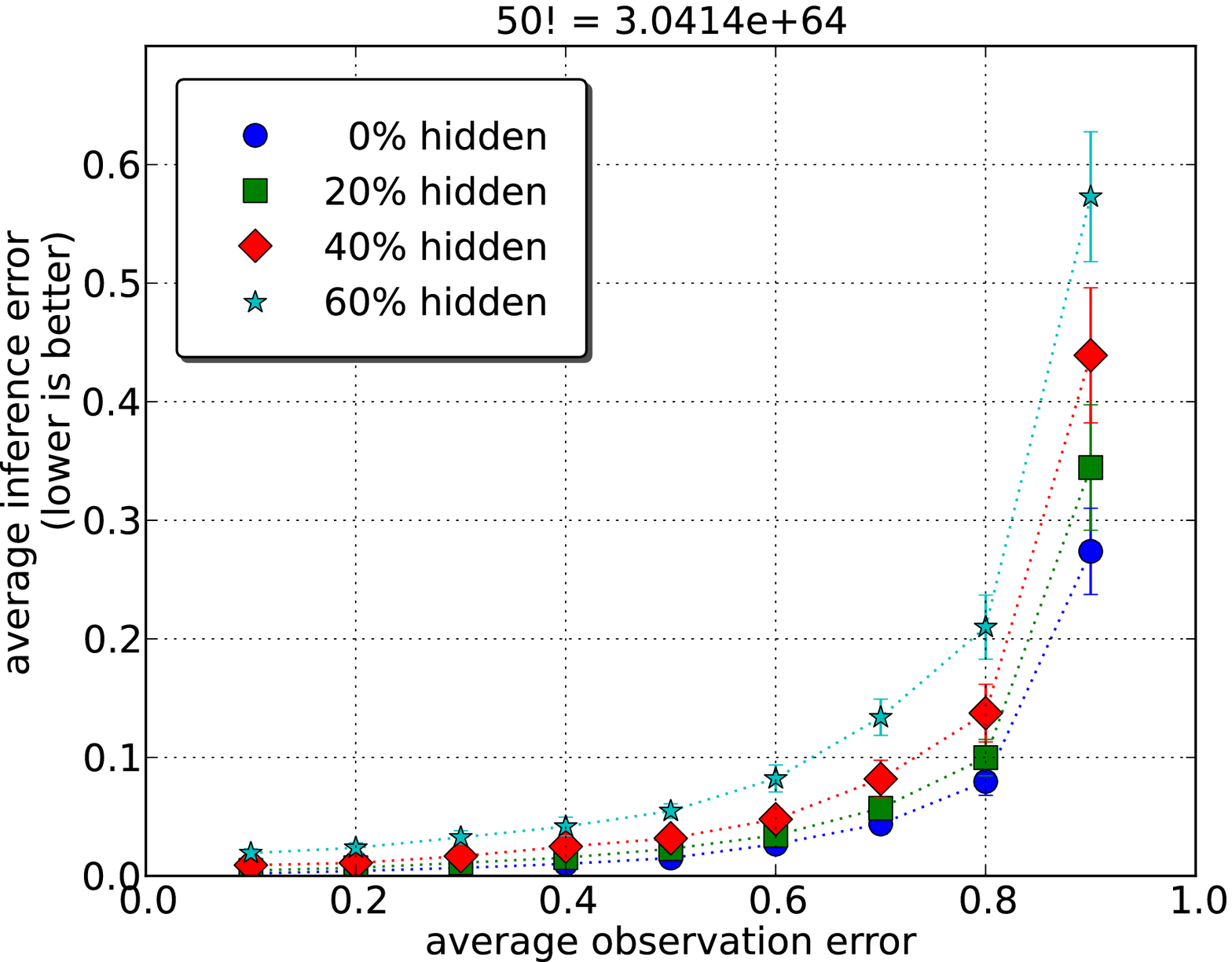}
  }
  \caption{Average error of a random hidden permutation inference from
    100 (partial)  noisy observations on  25 and 50  objects simulated
    datasets.   Runs   were  repeated   10  times  with   a  different
    permutation. }
  \label{fig:simulations}
\end{figure}
To demonstrate  correctness of  our approach, we  show inference  of a
fixed  hidden   permutation  from  its   noisy  partial  observations.
Figure~\ref{fig:simulations}  shows  results   of  this  inference  on
dataset  of  25 and  50  objects.   In  these first,  synthetic  data,
experiments,  we first  randomly  chose a  true (hidden)  permutation,
$\mathbf{P}_{true}$.     We   controlled   both    observation   noise
($\nu\in{0.1, 0.2,  ..., 0.9}$) and  fraction of objects  missing from
observations ($m\in{0\%, 20\%, 40\%, 60\%}$).  Noisy observations were
drawn      from     vMF($\mathbf{P}_{true}$,$\kappa_{\nu}$),     where
$\kappa_{\nu}$  was chosen  to achieve  $\nu$ fraction  of incorrectly
observed  object identities.   The final  observation, $\mathbf{P}_m$,
was  generated  by  hiding  $m$  percent of  entries  from  the  noisy
observation  matrix, chosen uniformly  at random  without replacement.
Figure~\ref{fig:simulations} shows that our representation of the $n!$
discrete  permutation  space  is   functional  and  the  approach  can
gracefully handle  large number  of objects, partial  observations and
observation noise.

\begin{figure}[ht!]
  \centering
  \subfloat[][tracking identities of 6 flights]{
    \includegraphics[width=0.43\textwidth]{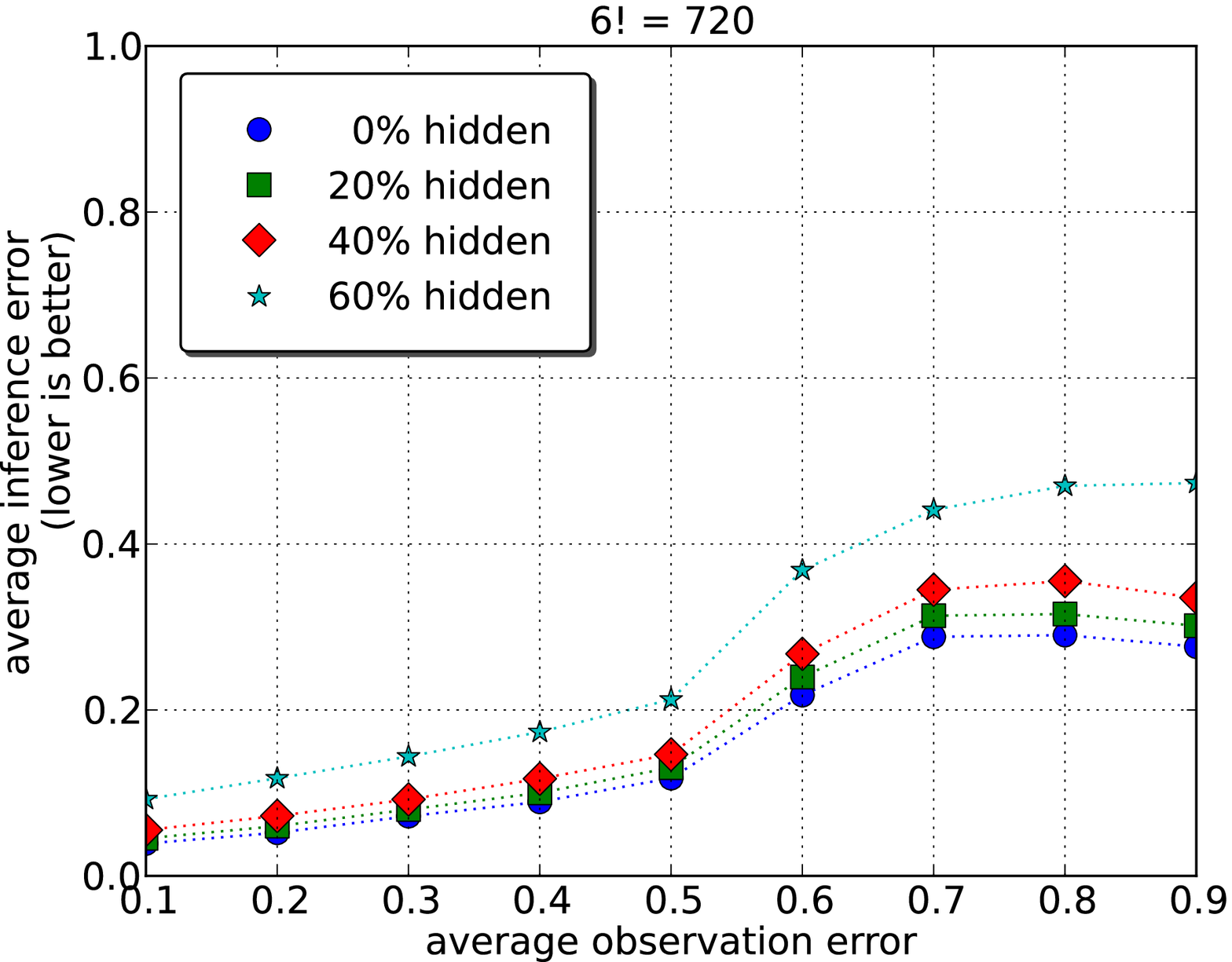}
  }
\qquad
  \subfloat[][tracking identities of 10 flights]{
    \includegraphics[width=0.43\textwidth]{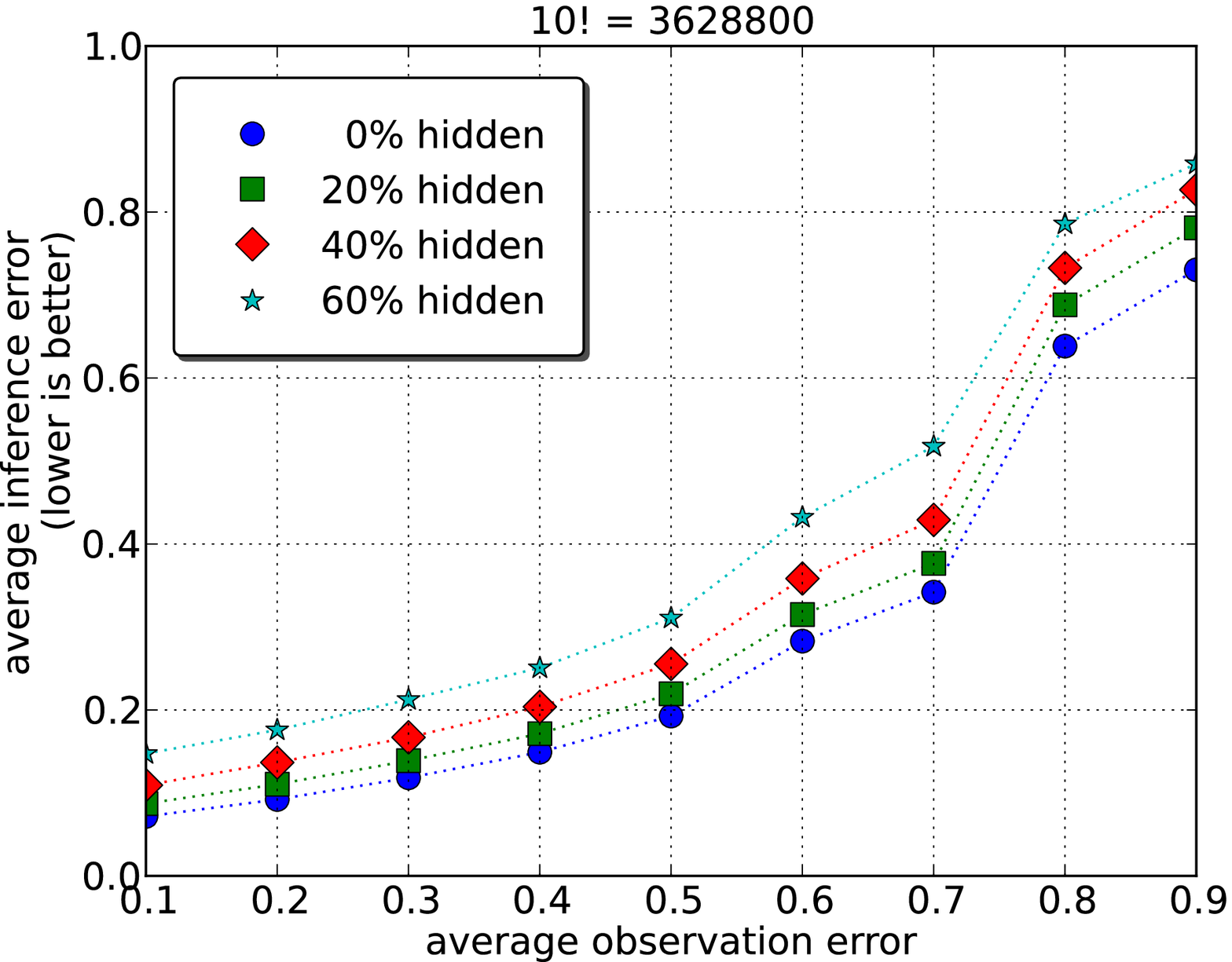}
  }
  \caption{Tracking error on the air traffic control dataset for 6 and
    10 planes as a function of observation noise shown as the fraction
    of  incorrectly reported  planes.  Separate  plots show  error for
    partial  observations  when a  fraction  of  object identities  is
    unobserved.}
  \label{fig:flights}
\end{figure}
The above  simulation was generated with  the noise model  used by the
inference  and did  not have  a  temporal component,  although it  was
applied to a really large state  space.  Next we show experiments on a
tracking dataset with a non-vMF  transition model. We use a dataset of
planar  locations of aircraft  within a  30 mile  diameter of  John F.
Kennedy  airport of  New  York.   The data,  in  streaming format,  is
available at {\tt  http://www4.passur.com/jfk.html}. The complexity of
the  plane routes  and  frequent  crossings of  tracks  in the  planar
projection  make this  an interesting  dataset for  identity tracking.
Identity  tracking results  on this  dataset,  in the  context of  the
symmetric semigroup approach to permutation inference, were previously
reported  in~\cite{KondorHowardJe}.   Replicating  the  task  reported
in~\cite{KondorHowardJe}, we  show results  on tracking datasets  of 6
and 10 flights, dropping the 15 flights dataset (but see below).

The  dataset comes prelabeled,  but the  uncertainty is  introduced by
randomly  swapping  identities  of   flights  $i$  and  $j$  at  their
respective  locations  $\vec{x}_i$  and $\vec{x}_j$  with  probability
$p_{swap}exp(-\|\vec{x}_j(t)    -   \vec{x}_i(t)\|^2/(2s^2))$,   where
$p_{swap}  = 0.1$  and $s  = 0.1$  are strength  and  scale parameters
respectively.

We then  generated observation and  hidden identity noise in  the same
way  as  for  the  prior experiment.   Figure~\ref{fig:flights}  shows
results of  applying our identity  tracking method to the  air traffic
control dataset for various levels  of observation noise and amount of
missing  identity  observations.   It  is  difficult  to  compare  the
performance to the method of~\cite{KondorHowardJe} applied to the same
dataset, since it is not clear how observation noise levels correspond
to each other. However, error values reported in~\cite{KondorHowardJe}
were 0.12 to 0.17  on the 6 flights dataset and 0.2  to 0.32 on the 10
flights dataset. This  is comparable to what we  get with our approach
for  observation  error below  50\%,  even  when  60\% of  the  flight
identities  are unobserved. Results  of the  application of  our state
space model  to this dataset indicate  robustness of the  model to the
choice  of  the  transition   model,  which  was  different  from  the
generative model of our tracking inference engine.

\begin{figure}[ht!]
  \centering
  \subfloat[][a frame from the tracking task]{
    \includegraphics[width=0.43\textwidth]{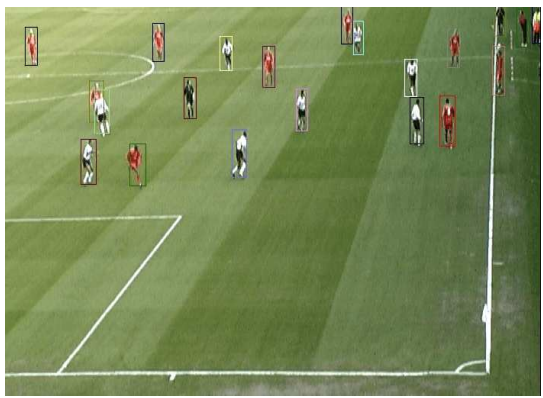}
  }
\qquad
  \subfloat[][tracking identities of 41 players]{
    \includegraphics[width=0.43\textwidth]{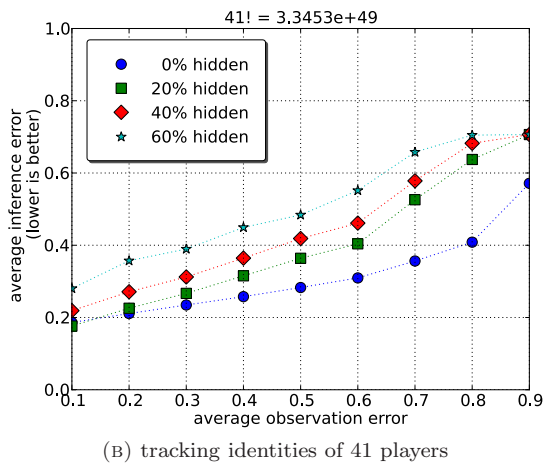}
  }
  \caption{Tracking  error on a  football visual  surveillance dataset
    for  41 players  as a  function of  observation  noise.  Different
    missing data fractions are shown.}
  \label{fig:soccer}
\end{figure}

Due  to the  unmanageable  size  of the  factorial  space in  identity
tracking problems,  even the powerful  and efficient methods  based on
Fourier representation  of permutations do not report  results on more
than     11~\cite{Huang_2009_6385}     or     15~\cite{KondorHowardJe}
simultaneously      tracked     objects.       The      results     of
Figure~\ref{fig:simulations} show  that our approach  can handle large
numbers   of   objects,   and  Figure~\ref{fig:flights}   demonstrated
comparable accuracy on the air  traffic control dataset.  Next we show
results  on 41 objects  from a  visual surveillance  dataset available
from            {\tt           http://vspets.visualsurveillance.org/}.
Figure~\ref{fig:soccer} shows  an example  of the underlying  data and
results of the identity tracking.  The problem is similar to the above
air  traffic  control:  we  have  added uncertainty  to  the  players,
identities  using  the same  exponential  proximity  model as  before.
Further, unlike the  air traffic domain, here there  are very few time
steps  that do  \emph{not} involve  an  identity swap.   This kind  of
situation is  difficult for  recursive Bayesian filtering  in general.
However, our  approach handles  the situation and  produces reasonable
results with acceptable error rate -- indeed, quite a good error rate,
considering the size of the state space.

\section{Conclusions}
The  main  result  of  this  work is  embedding  permutations  into  a
continuous manifold,  thus lifting a body of  results from directional
statistics   field~\cite{DirectionalStatistics}  to   the   fields  of
ranking,  identity  tracking   and  others,  where  permutations  play
essential role.   Among many potential applications  of this embedding
we have chosen probabilistic identity tracking and were able to set up
a  state-space model  with  efficient recursive  Bayesian filter  that
produced results comparable with the  state of the art techniques very
efficiently even  on a very  large datasets that pose  difficulties to
existing methods.   There remains much  to be done in  this direction.
However,  a simple  model,  that can  be  thought of  as a  continuous
generalization               of               the              Mallows
model~\cite{meila-consensus,fligner1986distance},     equipped    with
results  from  the field  of  directional  statistics has  efficiently
produced  results of  a reasonable  accuracy.  This  is  promising and
encourages  further   development  of  more   complicated  probability
distributions for permutations: further exploration of the exponential
family already developed  in the field~\cite{DirectionalStatistics} as
well  as  developing  more  complex  representations  using  spherical
harmonics representations.

\section*{Acknowledgments}
We thank Risi Kondor  for being impressively responsive and generously
providing the air traffic control  dataset. This work was supported by
NIH  under grant  number  NCRR  1P20 RR021938.  Dr.   Lane's work  was
supported by NSF under Grant No. 0705681.

\bibliography{papers}
\bibliographystyle{plain}

\end{document}